%% file: acl_latex.tex
\documentclass[11pt]{article}

\usepackage[final]{acl}

\usepackage{times}
\usepackage{latexsym}

\usepackage[T1]{fontenc}

\usepackage[utf8]{inputenc}

\usepackage{microtype}

\usepackage{inconsolata}

\usepackage{graphicx}

\usepackage[table]{xcolor} 
\usepackage{booktabs}      
\usepackage{subcaption}
\usepackage{graphicx}
\usepackage{multirow}      
\usepackage{graphicx}
\usepackage{multirow}
\usepackage{microtype}
\usepackage{graphicx}
\usepackage{subcaption}
\usepackage{verbatim}
\usepackage{booktabs} 
\usepackage{comment}
\usepackage{hyperref}

\usepackage{titlesec}

\usepackage{amsmath}
\usepackage{amssymb}
\usepackage{mathtools}
\usepackage{amsthm}
\usepackage{xcolor}

\theoremstyle{plain}
\newtheorem{theorem}{Theorem}[section]

\theoremstyle{definition}

\newtheorem{assumption}[theorem]{Assumption}
\theoremstyle{remark}

\newcommand{\xl}[1]{\textcolor{cyan}{\textbf{[xl: #1]}}}
\newcommand{\ours}{\textsc{EntroRouter}}
\newcommand{\sectionref}{\textsection}

\usepackage{booktabs}
\usepackage{multirow}
\usepackage[table]{xcolor}
\usepackage{makecell}
\usepackage{array}
\usepackage{xspace}
\usepackage{marvosym}

\usepackage[textsize=tiny]{todonotes}
\title{\ours: Learning Efficient Model Routing  \\ via Entropy Regulation}

\author{%
  Kaiyi Zhang$^{1,2}$\footnotemark[1],\quad Xueliang Zhao$^{3}$\footnotemark[1],\quad Zhuocheng Gong$^{4}$\\ \textbf{Wei Wu}$^{2}$\footnotemark[2],\quad \textbf{Yankai Lin}$^{1}$\footnotemark[2] \\
  $^{1}$Gaoling School of Artificial Intelligence, Renmin University of China \\
  $^{2}$Ant International, $^{3}$The University of Hong Kong, $^{4}$Ant Group \\
  \Letter~\texttt{\{kyzhang, yankailin\}@ruc.edu.cn},\quad \texttt{wuwei19850318@gmail.com}  \\
}

\begin{document}
\maketitle
\renewcommand{\thefootnote}{\fnsymbol{footnote}}
\footnotetext[1]{Equal Contribution.}
\footnotetext[2]{Corresponding author: Wei Wu and Yankai Lin.}
\input{0-abstract}

\input{1-introduction}
\input{2-method}
\input{3-experiments}
\input{4-related}

\input{5-conclusion}

\bibliography{custom}


\input{6-appendix}
\end{document}

%% file: 0-abstract.tex
\begin{abstract}
Model routing balances solution accuracy and computational cost by selecting among models of varying capabilities. While recent multi-round frameworks interleave reasoning and planning, we identify a structural failure mode termed Trust Region Collapse. We demonstrate that the deep coupling of reasoning and routing, exacerbated by the dominance of strong pre-training priors under sparse supervision, leads to degenerate local optima where capable experts are systematically suppressed. To decouple these processes, we propose \textsc{EntroRouter}, a single-round routing framework that treats entropy regulation as a core objective. We first initialize the policy via Soft Supervision, fitting a distribution of suitable models to establish a high-entropy prior for exploration. Subsequently, we stabilize Reinforcement Learning using a Soft Anchor, which utilizes offline capability estimates to orchestrate controlled entropy contraction within a safe trust region. Extensive experiments demonstrate that \ours~retains 98.3\% of the strongest expert's accuracy while reducing computational costs by 48.25\%.
\end{abstract}

%% file: 1-introduction.tex
\section{Introduction}

The rapid advancements in Large Language Models (LLMs) have introduced a critical trade-off in deployment: while larger models (e.g., Qwen3-235B) offer superior reasoning capabilities, they incur prohibitive computational costs and latency. Conversely, smaller models (e.g., Qwen3-4B) are efficient but often falter on complex tasks. This trade-off has necessitated the development of model routing systems, which aim to dynamically select the most cost-effective model for a given query without compromising solution quality~\cite{chen2023frugalgpt, vsakota2024fly}.

Current research explores two primary paradigms to solve this routing problem: iterative multi-round agents~\cite{zhang2025router} and deterministic single-round classifiers~\cite{ong2024routellm}. However, despite their architectural differences, we identify that both paradigms frequently suffer from a unified structural pathology: \textbf{Premature Entropy Collapse}. This failure to properly manage the policy's uncertainty, whether in the action space or the label space, prevents routers from exploring and learning the optimal cost-accuracy frontier.

The emerging multi-round paradigm (e.g., Router-R1) manifests as a collapse in the Action Space. Specifically, the router acts as an agent that decides when to query external experts. However, through empirical replication in the mathematical domain, we observe that the Reinforcement Learning (RL) process frequently drives the probability of calling external experts to zero (as shown in Figure~\ref{fig:collapse}). We demonstrate that this is a structural failure mode arising from the coupling of planning and execution. When a small router fails to comprehend the Out-Of-Distribution reasoning traces from larger experts, dominated by its strong pre-training priors, it effectively disregards the expert's response. Consequently, the RL optimizer perceives the expert call as a high-cost, zero-gain action, prematurely converging to a deterministic policy that refuses to route. We term this phenomenon Trust Region Collapse.

Parallel to this, traditional single-round approaches suffer from a corresponding collapse in the Label Space. This implies that in the final selection, the router is forced to prematurely commit to a rigid, deterministic choice. Existing baselines~\cite{ong2024routellm} rely on Hard Supervision (e.g., one-hot labels), forcing the router to mimic a single ``best'' model. This artificially suppresses the inherent ambiguity of problem difficulty. By forcing the target probability distribution to collapse to a sharp decision boundary too early during Supervised Fine-Tuning (SFT), the router loses the representational entropy required to capture nuanced trade-offs, rendering the policy prone to overfitting.

\begin{figure}[t]
  \begin{center}
    \centerline{\includegraphics[width=\columnwidth]{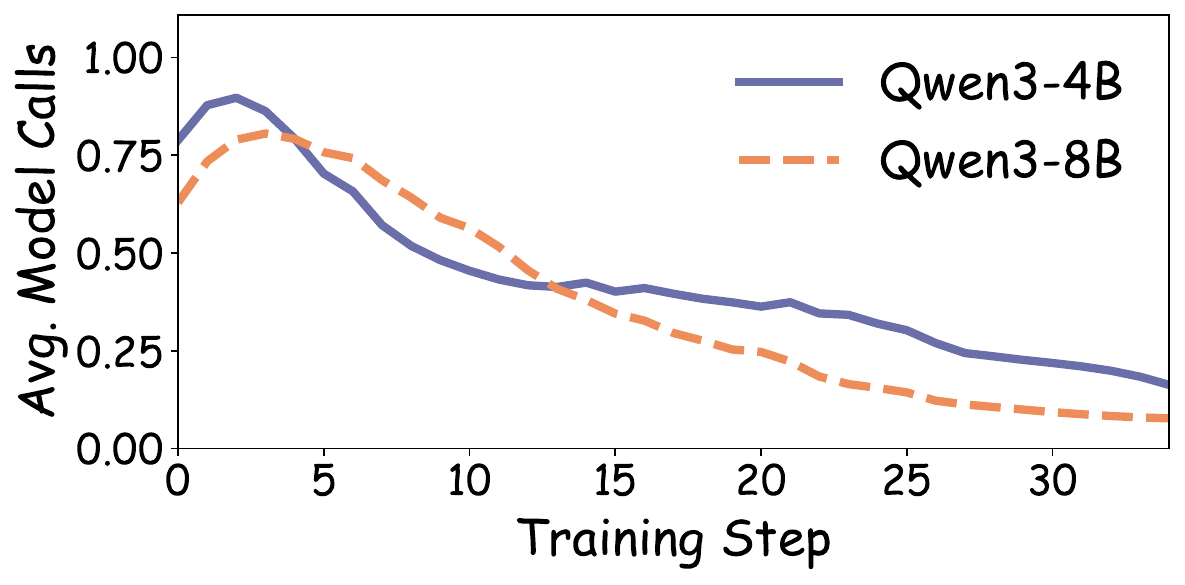}}
    \caption{Empirical observation of Trust Region Collapse. Despite successful cold-start fine-tuning, the subsequent RL process rapidly suppresses the probability of calling external models (Qwen3-4B/8B) to zero. This represents a premature entropy collapse, effectively abandoning the expert solver.}
    \label{fig:collapse}
  \end{center}
\end{figure}

To address these dual failures of entropy management, we propose \ours, a framework that treats entropy regulation as a first-class objective. We adopt the single-round routing paradigm and introduce a two-stage training pipeline designed to harmonize exploration and exploitation. First, to prevent the label-space collapse typical of hard baselines, we initialize the router via \textbf{Soft Supervision} (\sectionref~\ref{sec:stage1}). Instead of fitting a deterministic label, we train on a distribution of viable models. This establishes a high-entropy prior, preserving the policy's capacity to recognize difficulty ambiguity. Second, during the RL phase, we regularize the optimization using a Soft Anchor derived from offline capability estimates (\sectionref~\ref{sec:stage2}). Crucially, this mechanism serves a dual purpose: it defines a safe trust region to prevent the policy from drifting into degeneration, while simultaneously guiding the router to rapidly minimize entropy within this region. This allows the model to transition from the exploratory SFT state to a decisive, cost-efficient strategy that confidently selects reliable experts.

We extensively evaluate \ours~ on 7 rigorous mathematical reasoning benchmarks. Remarkably, \ours~ retains \textbf{98.3\%} of the strongest fixed-expert baseline's accuracy while reducing computational costs by \textbf{48.25\%}. Furthermore, our framework exhibits strong generalization to out-of-distribution tasks. On widely recognized benchmarks, \ours~ reduces costs by \textbf{40.5\%} while maintaining a relative accuracy within \textbf{95.4\%} of the strongest fixed-expert baseline.

Our main contributions are threefold. 
\textbf{First}, we formalize a sufficient condition under which multi-round routing suffers \textbf{Trust Region Collapse}, a failure regime where strong router priors suppress expert usage. \textbf{Second}, we propose \ours, a framework that orchestrates a \textbf{controlled entropy contraction} process. It utilizes Soft Supervision to establish a high-entropy exploration prior and employs a Soft Anchor to regularize RL optimization within a safe trust region. \textbf{Third}, we verify the effectiveness of our proposed method through extensive experiments on both mathematical benchmarks and out-of-distribution datasets.

%% file: 2-method.tex
\section{Preliminaries}
\label{sec:preliminaries}

In this section, we formalize the problem of LLM routing within the framework of a Markov Decision Process (MDP). We establish the notations used throughout the paper and review the prevailing multi-round routing paradigm, exemplified by Router-R1~\cite{zhang2025router}, which serves as the primary baseline for our analysis.

\subsection{Problem Formulation}
\label{pre:form}
We formulate the model routing task as a sequential decision-making process where a policy model (the \textit{Router}) selects an inference model (the \textit{Expert}) at each step to solve the user query. Let $x$ denote the input query, and let $\mathcal{M} = \{m_1, m_2, \dots, m_K\}$ represent the set of available candidate models that the router can choose from. We define this process as a tuple $\langle \mathcal{S}, \mathcal{A}, \mathcal{P}, \mathcal{R} \rangle$. This routing process generates a trajectory $\tau = (s_0, a_0, \dots, s_L)$ based on the policy and the selected models.

The components of the tuple are defined as follows. The \textbf{State Space ($\mathcal{S}$)} consists of states $s_t \in \mathcal{S}$ representing the accumulated context at time step $t$, initialized with the user query $s_0 = x$. The \textbf{Action Space ($\mathcal{A}$)} is defined as $\mathcal{A} = \mathcal{M} \cup \{m_{\text{self}}\}$, comprising both external candidate models and the router itself ($m_{\text{self}}$). Regarding \textbf{Transition Dynamics ($\mathcal{P}$)}, the state update depends on the selected action: if the router performs internal reasoning (i.e., $a_t = m_{\text{self}}$), it generates an output $o_t$, updating the state as $s_{t+1} = s_t \oplus o_t$, where $\oplus$ denotes sequence concatenation; for external delegation (i.e., $a_t \in \mathcal{M}$), the router formulates a query $q_t$, receives a response $o_t$ generated by the external model $a_t$, and updates the state to include the interaction as $s_{t+1} = s_t \oplus q_t \oplus o_t$.
Finally, the \textbf{Reward Function ($\mathcal{R}$)} balances solution correctness against computational cost. We define the total return of a trajectory as $R(\tau) = r_{\text{acc}} - \lambda \cdot C_{\text{total}}$, where $r_{\text{acc}} \in \{0, 1\}$ is a sparse binary indicator for solution correctness at the terminal state, and $C_{\text{total}} = \sum_{t} C(a_t) \cdot |o_t|$ represents the accumulated token-based cost, where $C(a_t)$ is the per-token cost of $a_t$.

The objective of the router is to learn a policy $\pi_\theta(a \mid s)$ that maximizes the expected return: $J(\theta) = \mathbb{E}_{\tau \sim \pi_\theta} [R(\tau)]$. In this formulation, $\lambda$ serves as a hyperparameter controlling the trade-off between accuracy and cost. We operate under the premise of a positive correlation between inference cost and model capability.  We assume that $\mathcal{M}$ consists exclusively of efficient frontier models, thereby precluding cases where a candidate is both more computationally expensive and functionally inferior to another. In the simplest scenario (single-step routing), the router selects a model based solely on the initial query $x$ without intermediate reasoning steps.

\subsection{The Multi-Round Routing Paradigm}

Recent advancements, such as Router-R1~\cite{zhang2025router}, generalize the routing task to a sequential setting. In this paradigm, the router operates as an agent interleaving internal reasoning steps ($a_t = m_{\text{self}}$) with external model calls ($a_t \in \mathcal{M}$). Unlike single-step routing where the router acts solely as a dispatcher, this framework requires the router to effectively \textit{interpret} and \textit{integrate} the external model's output $o_t$ to update the state and progress toward the solution.

Crucially, this formulation intertwines the \textit{routing decision} with the router's own \textit{reasoning execution}. Consequently, the policy's potential is bounded by the router's ability to assimilate complex information, a process often obstructed by the dominance of its own pre-training priors. As demonstrated in the following section, this dependency can lead to structural failures, where the router struggles to optimize the cost-accuracy trade-off due to stubborn internal priors and capability mismatches under sparse outcome reward.

\section{Methodology}

\begin{figure*}[t]
    \centering
    \begin{subfigure}[b]{0.36\textwidth} 
        \centering
        \includegraphics[height=5.5cm, keepaspectratio]{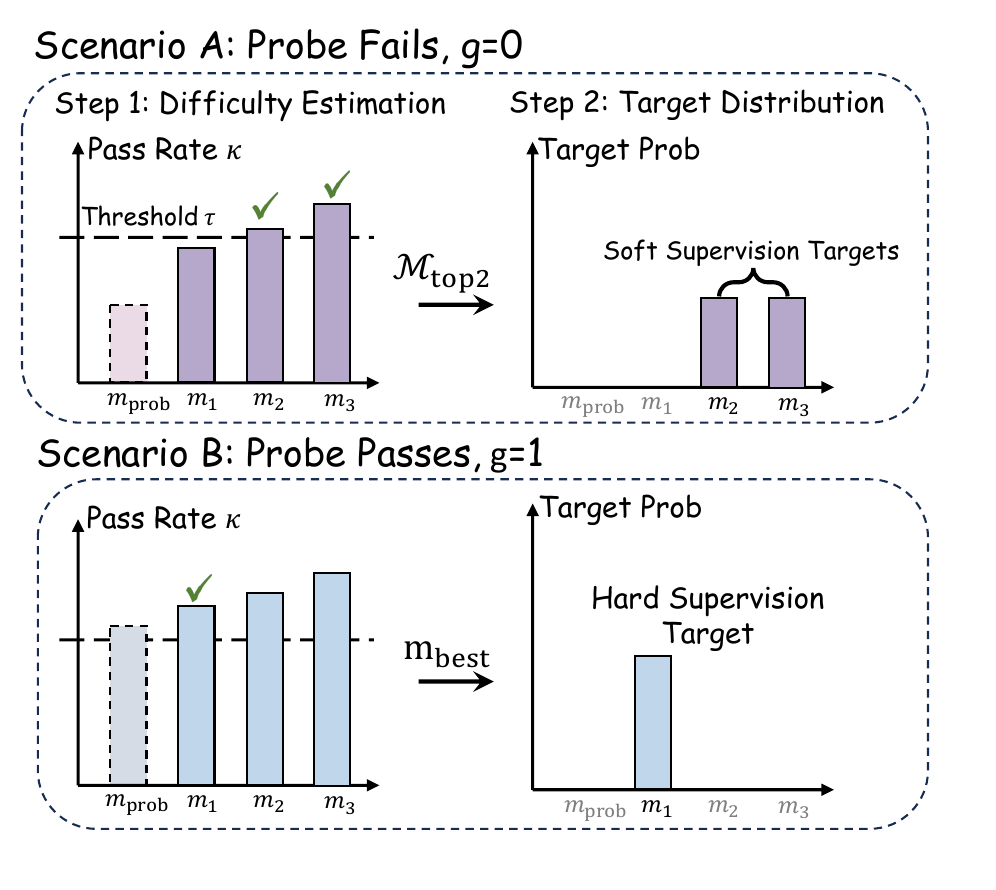}
        \caption{Stage I: SFT via Soft Supervision}
    \end{subfigure}
    \hfill
    \begin{subfigure}[b]{0.63\textwidth} 
        \centering
        \includegraphics[height=5.5cm, keepaspectratio]{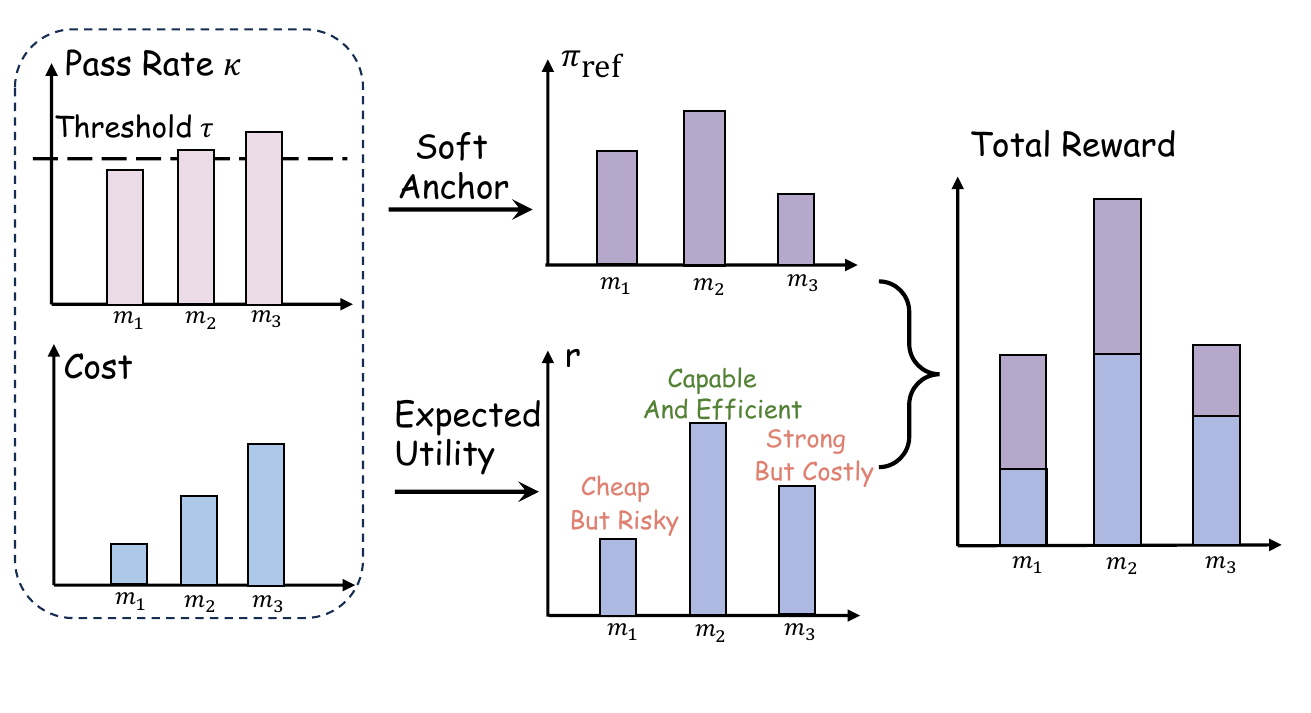}
        \caption{Stage II: Reinforcement Learning with Soft Anchor}
    \end{subfigure}
    
    \caption{\textbf{Overview of the \ours~Training Framework.} 
\textbf{(a) Stage I: SFT via Soft Supervision.} We employ a soft supervision strategy to prevent premature policy collapse. A weak probe $m_{\text{probe}}$ identifies query difficulty: for complex queries (Scenario A), we spread the training target across all qualified experts in $\mathcal{M}_{\text{topk}}$, deliberately maintaining a high-entropy policy to preserve exploration capacity. 
\textbf{(b) Stage II: Reinforcement Learning with Soft Anchor.} We refine the policy via GRPO using an expectation-grounded reward. By regularizing the update with a \textit{Soft Anchor} ($\pi_{\text{ref}}$), we constrain the router within a safe trust region while simultaneously penalizing entropy to transition from a broad exploration state to a decisive, high-confidence routing state.}
    \label{fig:SoftRouter_framework}
\end{figure*}

We begin by defining the theoretical failure mode observed in multi-round routing (\sectionref~\ref{sec:collapse_observation}). Motivated by this analysis, we propose \ours, a framework that treats routing as a single-round decision process. 
Specifically, our method is built upon a two-stage training pipeline: In the first stage, we initialize the policy by distributing probability mass across multiple viable models instead of assigning a single hard label. This \textit{entropy-preserving} initialization prevents premature convergence and maintains the capacity for further exploration (\sectionref~\ref{sec:stage1}). Subsequently, we fine-tune the policy via RL, regularizing updates against a reference distribution derived from offline capability estimates. This stage implements \textit{entropy contraction} within a safe trust region, dynamically optimizing the trade-off between solution correctness and computational cost while sharpening the final routing decision (\sectionref~\ref{sec:stage2}).

\subsection{Policy Degeneration in Multi-Round Routing}
\label{sec:collapse_observation}

We identify a critical instability that can arise in the multi-round routing paradigm under sparse outcome supervision and strong router pre-training priors.
Despite cold-start fine-tuning, the subsequent RL process frequently drives the probability of utilizing the strong model to zero, effectively abandoning the expert solver. Consequently, the policy's exploration frontier collapses into a degenerate local optimum of ``always refusing to escalate,'' a phenomenon we term \textit{Trust Region Collapse}. From an information-theoretic perspective, this manifests as a pathological and premature reduction in policy entropy: the router rapidly converges to a deterministic, low-entropy state of refusal, thereby eliminating the exploration variance required to discover better strategies.

We attribute this collapse to a functional conflation of planning and execution, which triggers a fatal causal chain characterized by three distinct phases. \textbf{First:} When the expert model generates a complex reasoning trace in response to the router's request, the smaller router often struggles to comprehend the logic because the trace is Out-Of-Distribution relative to its own training distribution. \textbf{Second:} Dominated by its own pre-training priors, the router effectively ignores the expert's provided context and falls back to generating the solution using its own limited capabilities. \textbf{Third:} Since the router has reverted to its own capabilities, the final outcome is no better than if it had acted alone. However, under sparse outcome based rewards, the RL optimizer observes a trajectory with \textit{high cost} (calling the expert) but \textit{no performance gain}. Lacking progress reward supervision to pinpoint the comprehension failure, the optimizer simply penalizes the escalation action.

Consequently, from a value-function perspective, the advantage of escalation turns negative. Since calling the expert incurs a computational cost but yields no improvement in accuracy, the optimizer perceives the expert call as strictly inferior to internal reasoning, forcing the policy to collapse.

\paragraph{Proposed Solution: Single-Round Routing.} 
The analysis above shows that the failure stems from forcing the small router to \textit{interact with and process the reasoning outputs of external expert models}. To alleviate this issue, we adopt \ours, a framework built upon the single-round routing paradigm, as the most direct and robust solution. Under \ours, the router decides which model to use based solely on the initial question. This formulation shifts the router's role: unlike the multi-round \textit{reasoner} that interacts with and comprehend the external solution, our router operates as a \textit{meta-analyzer}, estimating the query's difficulty to dispatch it to the most appropriate candidate model based solely on the initial question.

\subsection{Stage I: SFT via Soft Supervision}
\label{sec:stage1}

Our goal is to train a router policy $\pi_\theta( m \mid x)$\footnote{Following~\citet{zhao2025promptcot}, we equip the router with an intermediate reasoning step to explicitly assess query complexity before selection.
} that analyzes the query $x$ complexity via a reasoning trace before selecting the most suitable model $m$. To prevent the router from overfitting to a single deterministic choice, we introduce SFT with soft supervision. Specifically, we first identify the set of most economical and reliable experts based on offline capability estimates; then we construct a target distribution that adaptively switches between hard labels for simple queries and soft distributions for complex ones; finally, we optimize the router policy to minimize its divergence from this adaptive target. To teach the router to prioritize economy, we first determine which models are capable of solving a given query $x$. To do this, we estimate the ground-truth pass rate $\kappa(m, x)$ for each model $m$, which quantifies the likelihood that model $m$ correctly answers the query $x$. This pass rate is estimated via Monte Carlo rollouts, where $\kappa(m, x)$ is defined as the proportion of $N$ sampled responses from model $m$ that match the ground truth.

Based on this estimation, we formally define the optimal routing decisions. We operate under the general assumption that model capability correlates positively with inference cost. Under this premise, identifying the optimal expert simplifies to finding the most economical model that meets the reliability threshold $\tau$. 
Accordingly, we define the most economical expert $m_{\text{best}}$ as:

\begin{equation}
\label{eq:m_best}
    m_{\text{best}} = \operatorname*{arg\,min}_{m \in \mathcal{M}, \, \kappa(m, x) \geq \tau} C(m),
\end{equation}
where $C(m)$ denotes the per-token inference cost of model $m$. To construct a flexible training target, we further identify the set of the top-$k$ most cost-effective models, denoted as $\mathcal{M}_{\text{topk}}$. Under our assumption, this set corresponds to the $k$ cheapest models that effectively pass the threshold $\tau$:

\begin{equation}
\begin{aligned}
    \mathcal{M}_{\text{topk}} = \operatorname*{arg\,min}_{S \subseteq \mathcal{M}, \, |S| = k} & \sum_{m \in S} C(m) \\
    \text{s.t.} \quad & \forall m \in S, \, \kappa(m, x) \geq \tau.
\end{aligned}
\end{equation}

Intuitively, this selection criterion filters for the most economical experts among those sufficient for the task. Standard SFT typically forces the router to mimic a single ``best'' model, which works well for simple queries but fails to capture the inherent ambiguity in more complex queries, where model suitability is not straightforward. 
To explicitly model this ambiguity, we construct a target distribution $Q(m \mid x)$ that preserves the natural uncertainty of the task. We combine hard supervision (for trivial cases) with soft supervision (for complex cases) to prevent overfitting to arbitrary label noise:
\begin{equation}
\label{eq:unified_target}
\begin{split}
    Q(m \mid x) = & \underbrace{g(x) \cdot \mathbb{I}[m = m_{\text{best}}]}_{\text{Hard Supervision}} \\
    & + \underbrace{(1 - g(x)) \cdot \frac{\mathbb{I}[m \in \mathcal{M}_{\text{topk}}]}{k}}_{\text{Soft Supervision}}
\end{split}
\end{equation}
where $\mathbb{I}[\cdot]$ is the indicator function. 
The gating term $g(x)$ switches the supervision mode based on a lightweight probe model $m_{\text{probe}}$. For simple queries, where $g(x)=1$, if the probe demonstrates sufficient capability (i.e., $\kappa \ge \tau$), the query is deemed trivial, and the target collapses to the cheapest model in $\mathcal{M}$. For complex queries, where $g(x)=0$, the query is identified as complex, and we distribute probability mass uniformly across the valid set $\mathcal{M}_{\text{topk}}$, ensuring robustness against label noise. We optimize the router to match this target distribution. The loss function is defined as:
\begin{equation}
\label{eq:learning_objective}
    \mathcal{L}_{\text{SFT}}(\theta) = - \sum_{m \in \mathcal{M}} Q(m \mid x) \log \pi_\theta( m \mid x).
\end{equation}
By minimizing $\mathcal{L}_{\text{SFT}}$, the router learns to recognize simple cases for quick routing while maintaining a flexible decision boundary for complex queries.

\subsection{Stage II: Reinforcement Learning with Soft Anchor}
\label{sec:stage2}

While Stage I provides a reasonable initialization, the resulting policy remains static, ignoring the dynamic trade-off between correctness and computational efficiency. Stage II addresses this by learning a risk-cost preference through RL. However, unconstrained optimization risks destabilizing the policy. To enable controlled exploration, we propose a soft anchor mechanism. We implement this by incorporating a reference distribution derived from offline statistics directly into the optimization objective, effectively guiding the policy toward empirically safe and cost-effective regions.

We first construct the reference policy $\pi_{\text{ref}}$ to embody the ideal risk-cost preference. Based on the offline pass rate $\kappa(m, x)$ estimated in Stage I, we define a scalar utility metric $\mu(m, x)$ for each model $m \in \mathcal{M}$ given query $x$:
\begin{equation}
\label{eq:utility_metric}
    \mu(m, x) = \kappa(m, x) - \alpha \cdot C(m),
\end{equation}
where $C(m)$ is the inference cost and $\alpha$ governs the cost sensitivity. We then convert these utilities into a probabilistic prior using a Boltzmann distribution:
\begin{equation}
\label{eq:ref_policy}
    \pi_{\text{ref}}(m \mid x) = \frac{\exp(\mu(m, x) / T)}{\sum_{m' \in \mathcal{M}} \exp(\mu(m', x) / T)},
\end{equation}
where $T$ is a temperature hyperparameter that regulates the sharpness of the distribution. This distribution serves as the ``soft anchor'', assigning higher probability mass to models that are empirically reliable and economical. We optimize the router policy $\pi_\theta$ using Group Relative Policy Optimization (GRPO)~\cite{shao2024deepseekmathpushinglimitsmathematical}. Formally, our objective is to maximize the expected reward while regularizing the policy against the soft anchor:
\begin{equation}
\label{eq:theoretical_objective}
\begin{split}
    \max_{\pi_\theta} \, &\mathbb{E}_{m \sim \pi_\theta} [ r(m, x) ] \\
    &{} - \beta \mathbb{D}_{\text{KL}} [ \pi_\theta(\cdot|x) \,\|\, \pi_{\text{ref}}(\cdot|x) ] - \beta \mathcal{H}(\pi_\theta),
\end{split}
\end{equation}
where $\mathcal{H}(\pi_\theta)$ represents the entropy of the policy distribution. The task-performance reward $r(m, x)$ is defined as the \textit{expected utility}:
\begin{equation}
    r(m, x) = \kappa(m, x) R_{\text{succ}} + (1-\kappa(m, x)) R_{\text{fail}}.
\end{equation}
Here, $R_{\text{succ}}$ represents the cost-aware payoff received when the selected model provides a correct answer, while $R_{\text{fail}}$ is a failure penalty. The inclusion of $R_{\text{fail}}$ provides a negative gradient for high-risk candidates, actively suppressing the selection of incompetent experts. Our reward refines the raw reward $R(\tau)$ defined in Section~\ref{pre:form} in two key aspects. First, unlike $R(\tau)$ which relies on binary outcome, $r(m, x)$ utilizes the offline capability estimate $\kappa(m, x)$ as a dense, expectation-based signal to reduce training variance. Second, instead of a simple cost subtraction ($-\lambda \cdot C_{\text{total}}$), we embed the efficiency constraint into the shape of $R_{\text{succ}}$. This design enforces a cost-sensitive curriculum: for models cheaper than the optimal expert $m_{\text{best}}$ (which are presumed under-qualified under our assumption), we clamp the reward to prevent the policy from drifting toward unreliable models; conversely, for capable but expensive models, we apply a decay penalty to discourage unnecessary computational overhead (see Appendix~\ref{app:reward_details} for details).

Practically, we operationalize the objective in Eq.~\ref{eq:theoretical_objective} by incorporating the regularization term directly into the reward function, rather than modifying the underlying optimization algorithm. By augmenting the extrinsic task reward with the log-likelihood of the reference distribution, we implicitly enforce the dual regularization constraints. As formally derived in Appendix~\ref{app:prof_prior}, maximizing the return of this shaped reward is mathematically equivalent to the objective in Eq.~\ref{eq:theoretical_objective}. This formulation achieves a unique balance between safety and decisiveness through its regularization terms: 
(1) The $- \beta \mathbb{D}_{\text{KL}}$ term functions as a Soft Trust Region. Crucially, it regulates the dynamics of entropy reduction: by tethering the policy to the reference distribution $\pi_{\text{ref}}$, it prevents the negative entropy term from collapsing the policy into a degenerate, overly sharp state too abruptly. This ensures the router maintains valid exploration variance within the safe zone, preventing premature convergence to suboptimal local minima.
(2) The $- \beta \mathcal{H}(\pi_\theta)$ term encourages the router to minimize uncertainty in its final choice. In our reward-shaping implementation, this term is not applied as an additional standalone loss, it is algebraically absorbed into the reference-shaped reward. Once the reasoning process identifies a suitable model, this penalty incentivizes the router to produce a "sharp" and confident decision, preventing the policy from remaining in an indecisive probabilistic spread.

\begin{table*}[!t] 
\begin{center} 
\renewcommand{\arraystretch}{1.2} 
\tabcolsep=0.008\linewidth 
\resizebox{\linewidth}{!}{ \begin{tabular}{ l cc cc cc cc cc cc cc cc } \toprule \multirow{2}{*}{\textbf{Model}} & \multicolumn{2}{c}{\textbf{HMMT Nov 25}} & \multicolumn{2}{c}{\textbf{AIME25}} & \multicolumn{2}{c}{\textbf{OlympiadBench}} & \multicolumn{2}{c}{\textbf{AIME-24}} & \multicolumn{2}{c}{\textbf{AMC}} & \multicolumn{2}{c}{\textbf{MATH500}} & \multicolumn{2}{c}{\textbf{SMT2025}} & \multicolumn{2}{c}{\textbf{Avg.}} \\ \cmidrule(lr){2-3} \cmidrule(lr){4-5} \cmidrule(lr){6-7} \cmidrule(lr){8-9} \cmidrule(lr){10-11} \cmidrule(lr){12-13} \cmidrule(lr){14-15} \cmidrule(lr){16-17} & \textbf{Acc. ($\uparrow$)} & \textbf{Cost ($\downarrow$)} & \textbf{Acc. ($\uparrow$)} & \textbf{Cost ($\downarrow$)} & \textbf{Acc. ($\uparrow$)} & \textbf{Cost ($\downarrow$)} & \textbf{Acc. ($\uparrow$)} & \textbf{Cost ($\downarrow$)} & \textbf{Acc. ($\uparrow$)} & \textbf{Cost ($\downarrow$)} & \textbf{Acc. ($\uparrow$)} & \textbf{Cost ($\downarrow$)} & \textbf{Acc. ($\uparrow$)} & \textbf{Cost ($\downarrow$)} & \textbf{Acc. ($\uparrow$)} & \textbf{Cost ($\downarrow$)} \\ \midrule \rowcolor[gray]{0.95} \multicolumn{17}{c}{\textit{Heuristic}} \\
Random & 68.8 & 0.2783 & 73.3 & 0.2809 & 80.3 & 3.9836 & 84.6 & 0.2179 & 95.3 & 0.4213 & 98.0 & 1.2790 & 76.4 & 0.4956 & 87.1 & 0.0050 \\
Qwen3-4B & 59.0 & 0.0198 & 67.1 & 0.0192 & 77.5 & 0.2895 & 76.9 & 0.0173 & 93.8 & 0.0310 & 97.4 & 0.0907 & 65.3 & 0.0322 & 84.5 & 0.0004 \\
Qwen3-30B & 71.7 & 0.2488 & 72.1 & 0.2519 & 80.2 & 3.8433 & 87.5 & 0.2126 & 95.6 & 0.3893 & 98.2 & 1.2123 & 76.2 & 0.4983 & 87.2 & 0.0048 \\ 
Qwen3-235B & 81.5 & 0.5690 & 81.7 & 0.6179 & 83.6 & 8.7811 & 93.3 & 0.4553 & 97.6 & 0.8389 & 99.2 & 2.5881 & 84.2 & 0.9183 & 90.1 & 0.0105 \\ \midrule \rowcolor[gray]{0.95} \multicolumn{17}{c}{\textit{Existing Routing Architectures}} \\ 
RouteLLM & 59.0 & 0.0559 & 67.5 & 0.0408 & 78.2 & 1.7300 & 80.4 & 0.0433 & 95.8 & 0.0848 & 97.6 & 0.1042 & 69.8 & 0.0712 & 85.3 & 0.0015 \\ 
FORC & 70.4 & 0.2731 & 75.2 & 0.2647 & 80.9 & 4.6000 & 90.0 & 0.3435 & 96.2 & 0.5295 & 98.2 & 1.7675 & 77.6 & 0.5381 & 87.7 & 0.0059 \\ 
Router-R1 & 21.3 & 0.0108 & 28.8 & 0.0116 & 60.3 & 0.2174 & 41.7 & 0.0198 & 72.7 & 0.0241 & 93.4 & 0.0018 & 35.4 & 0.0359 & 70.0 & 0.0002 \\ 
xRouter & 28.8 & 0.0041 & 31.0 & 0.0033 & 60.9 & 0.0668 & 41.7 & 0.0018 & 72.3 & 0.0010 & 93.2 & 0.0009 & 34.4 & 0.0069 & 70.4 & 0.0001 \\ \midrule 
\textbf{\ours} & \textbf{80.0} & 0.4163 & \textbf{78.8} & 0.2929 & \textbf{81.8} & 5.2397 & \textbf{90.6} & 0.2881 & \textbf{96.6} & 0.3658 & \textbf{98.4} & 0.5238 & \textbf{80.4} & 0.5155 & \textbf{88.6} & 0.0055 \\ \bottomrule \end{tabular} } 
\end{center} 
\caption{ Comparison across math benchmarks. \textbf{Avg.} aggregates performance on math tasks. \textbf{Cost} is measured as the \textbf{total cost~($\$$)} for individual benchmarks, but converted to \textbf{average per-query cost} for the \textbf{Avg.} column. The best routing method excluding the strongest fixed expert is highlighted in \textbf{bold}.} 
\label{tab:main} 
\end{table*}

%% file: 3-experiments.tex
\section{Experiments}
\subsection{Setup}

\paragraph{Benchmarks.}
We evaluate ENTROROUTER primarily on mathematical reasoning, where instances naturally span a wide range of difficulty levels and thus provide a suitable testbed for cost-aware routing. Our evaluation suite includes \textbf{AIME 24}~\cite{aime24}, \textbf{AIME 25}~\cite{aime25}, \textbf{HMMT Nov 25}~\cite{balunovic_srimatharena_2025}, \textbf{SMT 25}, \textbf{MATH500}~\cite{hendrycks2021measuring}, \textbf{AMC}~\cite{li2024numinamath}, and \textbf{OlympiadBench}~\cite{he-etal-2024-olympiadbench}.

\paragraph{Candidate Models and Costs.}
The candidate pool $\mathcal{M}$ contains three Qwen3 reasoning models~\cite{yang2025qwen3technicalreport}: Qwen3-4B-Thinking-2507, Qwen3-30B-A3B-Thinking-2507, and Qwen3-235B-A22B-Thinking-2507. We abbreviate them as Qwen3-4B, Qwen3-30B, and Qwen3-235B. The largest model serves as the strongest-expert reference. For brevity, we also refer to it as the strongest fixed-expert baseline. Following prior routing work, we use standard API pricing as the cost function $C(m)$.\footnote{\url{https://novita.ai/pricing}} Router inference cost is excluded since its average overhead is only $1.7\times 10^{-5}$ dollars per query, which is negligible relative to expert-model costs.

\paragraph{Training Details and Baselines.}
We use the non-thinking \texttt{Qwen3-4B} as the router backbone. Following the data construction pipeline in Appendix~\ref{app:data_construction}, we first perform SFT and then optimize the router with RL. Full hyperparameters are provided in Appendix~\ref{app:train_details}. We compare against heuristic strategies, including random routing and fixed model selection, as well as representative routing methods: \textbf{RouteLLM}~\cite{ong2024routellm}, \textbf{FORC}~\cite{vsakota2024fly}, \textbf{xRouter}~\citep{qian2025xroutertrainingcostawarellms}, and \textbf{Router-R1}~\citep{zhang2025router}. Additional evaluation details are reported in Appendix~\ref{app:eval_details}.

\subsection{Main Results}
\label{sec:main_result}
The results on mathematical reasoning benchmarks are summarized in Table~\ref{tab:main}. \ours~achieves an average accuracy of \textbf{88.62\%}, outperforming all routing baselines and approaching the strongest fixed-expert baseline while substantially reducing cost. The results for Router-R1 and xRouter provide empirical evidence for the \textit{Trust Region Collapse} phenomenon analyzed in Section~\ref{sec:collapse_observation}. Both multi-round architectures \textit{succumb to} this instability, where the policy degenerates into rarely invoking external experts and instead relying exclusively on its own limited internal capabilities. Consequently, they yield poor performance, confirming that coupled reasoning-routing fails to bridge the capability gap under sparse supervision. We provide an additional intuitive explanation for this failure mode in Appendix~\ref{sec:router_fail}. By treating routing as a distinct meta-analysis task rather than an interleaved reasoning step, our framework avoids the interference of pre-training priors. Consequently, \ours~successfully identifies the cost-accuracy boundary, recovering \textbf{98.3\%} of the strongest expert’s accuracy while reducing computational costs by \textbf{48.25\%}.

\section{Discussion}

Beyond the main cost--accuracy results, we further analyze \textsc{EntroRouter} from the perspectives of component effectiveness, routing behavior, and robustness. We first ablate the key design choices in our training pipeline, including Soft Supervision, the gating term, the Soft Trust Region, and the performance reward (\sectionref~\ref{sec:ablation}). We then inspect the learned policy by visualizing routing distributions across problem-difficulty levels (\sectionref~\ref{sec:routing_analysis}). Additional studies in the appendix evaluate OOD generalization (Appendix~\ref{sec:OOD}), sensitivity to the cost--accuracy coefficient $\alpha$ (Appendix~\ref{sec:ablation_alpha}), transfer to heterogeneous candidate pools (Appendix~\ref{app:cross_family}), and robustness to noisy offline estimates (Appendix~\ref{app:noise_robustness}).

\begin{table}[t]
\centering

\renewcommand{\arraystretch}{1.15}

\resizebox{\linewidth}{!}{
\begin{tabular}{l cc cc cc cc}
\toprule
\multirow{2}{*}{\textbf{Method}} & \multicolumn{2}{c}{\textbf{HMMT-Nov}} & \multicolumn{2}{c}{\textbf{AIME25}} & \multicolumn{2}{c}{\textbf{SMT2025}} & \multicolumn{2}{c}{\textbf{Avg.}} \\
\cmidrule(lr){2-3} \cmidrule(lr){4-5} \cmidrule(lr){6-7} \cmidrule(lr){8-9}
& \textbf{Acc} & \textbf{Cost} & \textbf{Acc} & \textbf{Cost} & \textbf{Acc} & \textbf{Cost} & \textbf{Acc} & \textbf{Cost} \\
\midrule
\textbf{\ours} & \textbf{80.0} & 0.416 & \textbf{78.8} & 0.293 & \textbf{80.4} & 0.516 & \textbf{79.9} & 0.011 \\
\midrule
\ w/o Soft Supervision & 72.1 & 0.294 & 72.5 & 0.207 & 73.1 & 0.357 & 72.7 & 0.008 \\
\ w/o Gate term $g(\cdot)$ & 78.1 & 0.431 & 77.7 & 0.429 & 80.2 & 0.645 & 79.0 & 0.013 \\
\ w/o Soft Trust Region & 71.7 & 0.280 & 75.8 & 0.280 & 76.7 & 0.551 & 75.1 & 0.010 \\
\ w/o Performance Reward & 69.6 & 0.385 & 78.3 & 0.365 & 74.8 & 0.384 & 74.4 & 0.010 \\ 

\bottomrule
\end{tabular}
}
\caption{\textbf{Ablation Study.} Impact of removing Soft Supervision, Gating term $g(\cdot)$ , Soft Trust Regions, and Task-performance Reward. }
\label{tab:ablation}
\end{table}

\subsection{Ablation Study}
\label{sec:ablation}

We evaluate the contribution of each component in \ours~by comparing it against four variants: (1) replacing \textbf{Soft Supervision} with hard targets; (2) removing the \textbf{Gating term $g(\cdot)$} from the SFT target construction; (3) disabling the \textbf{Soft Trust Region} anchor; and (4) substituting our \textbf{Task-performance Reward} with sparse outcome rewards (See Appendix~\ref{app:outcome} for details). Results in Table~\ref{tab:ablation} indicate that \textbf{Soft Supervision} is paramount; its removal triggers a significant accuracy drop, confirming that hard supervision hinders exploration. The \textbf{Gating term $g(\cdot)$} is essential for cost-efficiency; by enforcing hard supervision for trivial instances during SFT, it explicitly teaches the router to assign simple queries to the cheapest expert. Without it, average costs increase as the policy fails to learn this strict economy for easy tasks. Finally, ablating the \textbf{Soft Trust Region} or \textbf{Performance Reward} degrades the trade-off, validating the necessity of anchored regularization and granular feedback for stable optimization.

\subsection{Analysis of Routing Distribution}
\label{sec:routing_analysis}

To verify whether \textsc{EntroRouter} successfully aligns model capabilities with problem complexity, we examine the routing distribution across different difficulty strata. We stratify the evaluation set into three levels based on $m_\text{best}$: \textit{Easy} ( $m_\text{best}=$ Qwen3-4B), \textit{Medium} ( $m_\text{best}=$ Qwen3-30B-A3B), and \textit{Hard} (remaining queries). As illustrated in Figure~\ref{fig:difficulty_routing}, the learned policy exhibits a clear adaptive strategy. For queries classified as \textit{Easy}, the router prioritizes computational efficiency, dispatching 76.0\% of instances to the lightweight Qwen3-4B model. Conversely, as problem complexity increases, the policy dynamically shifts probability mass towards stronger experts. In the \textit{Hard} stratum, reliance on the efficient 4B model contracts to 25.0\%, while the invocation of Qwen3-235B escalates significantly to 66.0\%. This distribution results in a weighted average API cost that rises monotonically with difficulty, confirming that \textsc{EntroRouter} has effectively internalized the latent relationship between query hardness and necessary reasoning capability. 

This adaptive distribution empirically validates the efficacy of \ours's entropy regulation mechanisms. While conventional routing systems often suffer from uncontrolled entropy decay, our framework orchestrates a precise management of exploration capacity. The Soft Supervision initializes the policy with a high-entropy prior to ensure diversity. Subsequently, in the RL stage, the reference-shaped reward encourages cost-aware decision sharpening, while the Soft Anchor constrains this contraction process. This balanced regulation ensures the router transitions to a decisive state without suffering from precipitous trust region collapse, retaining the probabilistic capacity to recognize and respond to high-complexity queries.

\begin{figure}[t]
    \centering
    \includegraphics[width=1.0\linewidth]{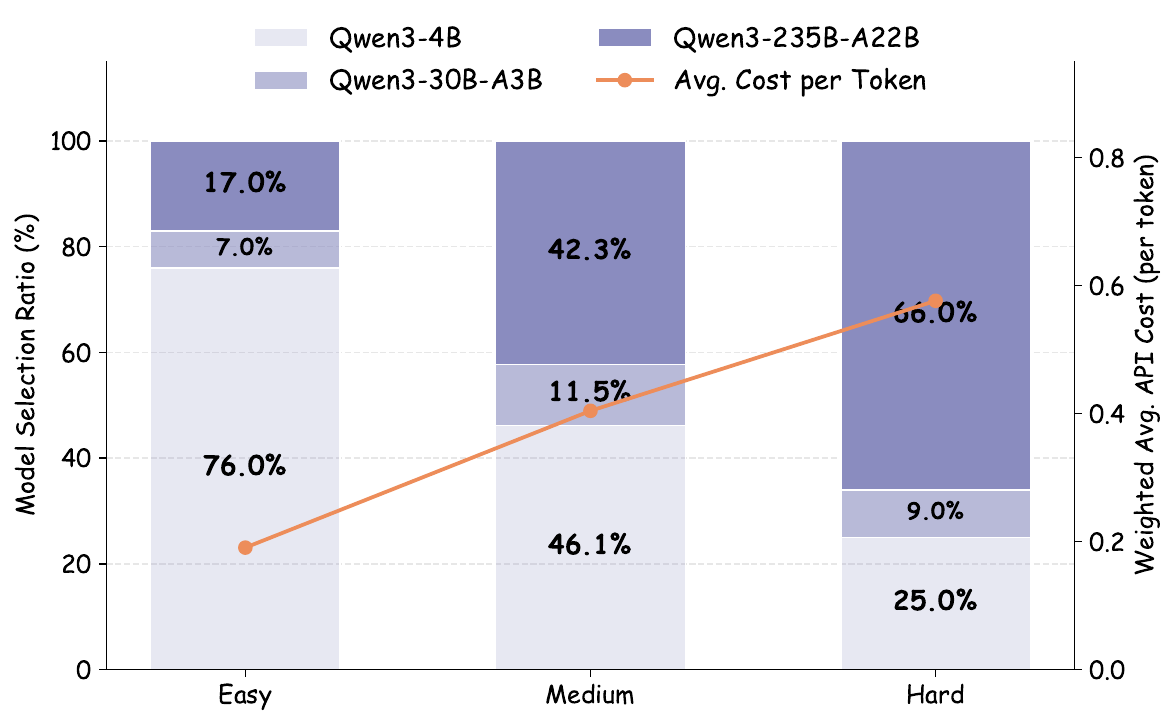} 
    \caption{\textbf{Routing Distribution \& Token Cost Efficiency across Difficulty Levels.} The stacked bars (left axis) represent the selection ratio of candidate models, while the trend line (right axis) tracks the weighted average API cost. \textsc{EntroRouter} demonstrates adaptive behavior by escalating model size in response to increasing problem difficulty.}
    \label{fig:difficulty_routing}
\end{figure}

%% file: 4-related.tex
\section{Related Work}
\paragraph{LLM Routing Systems}
The rapid proliferation of Large Language Models (LLMs) has necessitated the development of sophisticated routing mechanisms to balance performance, cost, and latency. Early approaches primarily focused on optimizing the trade-off between cost and accuracy through heuristic or predictive means. \citet{chen2023frugalgpt} sequentially invokes models from smaller to larger scales to reduce inference costs while maintaining performance. To standardize the evaluation of these growing routing strategies,  RouterBench~\cite{hu2024routerbench} was introduced as a comprehensive benchmark. Structurally, \citet{feng2024graphrouter} constructs interaction graphs between queries and models to capture complex dependencies that linear classifiers might miss.  More recently, \citet{ding2025best} introduced a test-time optimal compute framework, adaptively allocating computational resources based on the instance-level difficulty. Inspired by the success of RL in reasoning, recent research has formalized routing as a multi-round sequential decision process, where the router acts as an agent to iteratively query and integrate feedback from external experts~\cite{zhang2025router,qian2025xroutertrainingcostawarellms}.

\paragraph{Reinforcement Learning for Reasoning}

Reinforcement Learning has emerged as a critical paradigm to stimulate model reasoning without extensive supervision. Traditional RL approaches, such as Proximal Policy Optimization (PPO)~\cite{schulman2017proximal}, rely on a value function and a clipped surrogate objective to stabilize training within a trust region. More recently, Group Relative Policy Optimization (GRPO)~\cite{shao2024deepseekmathpushinglimitsmathematical} has gained prominence in mathematical reasoning. By estimating the baseline from the group mean of outputs rather than a critic, GRPO significantly improves training efficiency.

%% file: 5-conclusion.tex
\section{Conclusion}
In this work, we identified \textit{Trust Region Collapse} as a critical failure mode in the emerging multi-round routing paradigm, stemming from the capability mismatch between the router and expert models. To address this, we proposed \ours, a single-round framework that decouples planning from reasoning. By leveraging Soft Supervision and Soft Anchor regularization, our method establishes a stable optimization landscape that prevents policy degeneration. Empirical results demonstrate that \ours~ achieves Pareto efficiency, retaining \textbf{98.3\%} of the strongest expert’s accuracy while reducing computational costs by \textbf{48.2\%}. These findings suggest that small models are best utilized not merely as weaker reasoners, but as effective \textit{meta-controllers} for scalable AI systems.

\section*{Limitations}
This work focuses on establishing a robust and efficient routing framework under a single-round decision paradigm. Our analysis of Trust Region Collapse characterizes an important failure regime of multi-round routing, where expert calls are costly, supervision is sparse, and the router may not effectively use out-of-distribution expert traces. Future work could extend this analysis to richer multi-step routing agents equipped with stronger process supervision, explicit verification, or more capable aggregation modules. Similarly, while \ours~uses offline capability estimates $\kappa(m,x)$ to construct training targets and rewards, these estimates may contain sampling noise or verifier errors. We have examined this issue through noisy-estimate perturbation experiments, and a promising future direction is to develop adaptive profiling or periodic re-estimation strategies for dynamically changing model pools.

%% file: 6-appendix.tex
\newpage
\appendix

\section{Theoretical Analysis of Trust Region Collapse}
\label{app:theoretical_analysis}

In this appendix, we provide a formal derivation of the \textit{Trust Region Collapse} phenomenon. We first define the simplified decision paradigm and the decomposed advantage function. We then justify our core assumption based on the inverse scaling hypothesis and provide a rigorous proof that the effective trust region collapses to an empty set under sparse supervision.

\subsection{Problem Formulation and Advantage Decomposition}
To facilitate our derivation, we abstract the router's action space into a binary decision paradigm at step $t$:
\begin{enumerate}
    \item \textbf{Internal Reasoning ($a_{\text{self}}$):} The router generates the solution using its own parameters ($a_t = m_{\text{self}}$), incurring zero marginal cost.
    \item \textbf{External Escalation ($a_{\text{ext}}$):} The router delegates to the model pool ($a_t \in \mathcal{M}_{\text{pool}}$). We treat this as a unified, costly action with cost $C(a_{\text{ext}}) = c > 0$.
\end{enumerate}

The policy optimization is driven by the Advantage Function $A^\pi(s, a)$. Let $V^\pi(s)$ be the value function and $\bar{C}_\pi(s)$ the expected baseline cost. The advantage of the costly escalation action $a_{\text{ext}}$ decomposes into a performance gain and a cost penalty:
\begin{equation}
\label{eq:app_advantage}
\begin{split}
    A^\pi(s, a_{\text{ext}}) &= \underbrace{(\mathbb{E}[R \mid s, a_{\text{ext}}] - \mathbb{E}_{a \sim \pi}[R \mid s, a])}_{\text{Performance Gain } (\Delta R)} \\
    &\quad - \underbrace{(c - \bar{C}_\pi(s))}_{\text{Marginal Cost } (\Delta C)}
\end{split}
\end{equation}
Since $a_{\text{ext}}$ is strictly more expensive than $a_{\text{self}}$, the marginal cost term $\Delta C$ remains strictly positive unless the policy has fully converged to $a_{\text{ext}}$.

\subsection{Justification of Strong Prior Dominance}
\label{app:assumption_justification}
We posit that augmenting a weak router with a strong expert does not monotonically improve reward under sparse supervision. We formalize this as:

\begin{assumption}[Strong Prior Dominance]
\label{ass:strong_prior_app}
Let $\pi_{\text{router}}$ be initialized from a pre-trained LM and optimized via sparse outcome rewards. There exists a regime where the posterior update induced by an out-of-distribution expert trace $o\sim \pi_{\text{expert}}$ is dominated by the router's prior, such that:
\begin{align}
\mathbb{E}_{o \sim \pi_{\text{expert}}(\cdot \mid x)}[R \mid x \oplus o] \le \mathbb{E}[R \mid x].
\end{align}
\end{assumption}

\paragraph{Theoretical Grounding.}
This assumption is grounded in the ``Inverse Scaling'' taxonomy~\cite{mckenzie2023inverse}. Large language models rely on two information sources: pre-training priors and in-context prompts. When an expert provides a complex, out-of-distribution (OOD) reasoning path, it creates a conflict. 
\begin{itemize}
    \item \textbf{Strong Priors and Unwanted Imitation:} As seen in \textit{Resisting Correction} tasks, models often ignore OOD instructions (the expert trace) and revert to their own high-likelihood generation manifold. The router essentially ``talks over'' the expert.
    \item \textbf{Distractor Task Failure:} Under sparse supervision, the optimizer cannot distinguish between \textit{planning failure} (wrong routing) and \textit{execution failure} (misinterpreting the expert). The expert trace effectively acts as a ``distractor,'' causing the router to focus on superficial pattern matching rather than the intended logic, leading to $\Delta R \le 0$.
\end{itemize}

\subsection{Proof of Trust Region Collapse}
\label{app:proof_collapse}

We formally define the \textit{Effective Trust Region} $\mathcal{T}_{\text{eff}}$ as the set of policy distributions that satisfy safety constraints, positive expected advantage, and an increased probability of utilizing the expert:

\begin{equation}
\label{eq:teff_def}
\mathcal{T}_{\text{eff}} \triangleq \left\{ \pi' \in \Delta(\mathcal{A}) \;\middle|\; 
\begin{aligned}
    & D_{KL}(\pi' \,||\, \pi) \le \epsilon, \\
    & \quad (\text{Trust Region}) \\
    & \mathbb{E}_{a \sim \pi'}[A^\pi(s, a)] > 0, \\
    & \quad (\text{Policy Improvement}) \\
    & \pi'(a_{\text{ext}}) > \pi(a_{\text{ext}}) \\
    & \quad (\text{Escalation Intent})
\end{aligned}
\right\}
\end{equation}

Based on the \textit{Performance Difference Lemma}, this region represents the viable search space where increasing reliance on the strong solver is theoretically guaranteed to improve the objective.

\begin{theorem}[Trust Region Collapse]
Under Assumption~\ref{ass:strong_prior_app} and strict positive cost $c > 0$, the advantage of escalation is strictly negative ($A^\pi(s, a_{\text{ext}}) < 0$). Consequently, $\mathcal{T}_{\text{eff}} = \emptyset$.
\end{theorem}

\begin{proof}
The proof proceeds by contradiction. We first establish the sign of the advantage function and then analyze the constraints on $\pi'$.

\textbf{1. Negative Advantage of Escalation.}
By Assumption~\ref{ass:strong_prior_app}, the expected reward of escalation satisfies $\mathbb{E}[R|a_{\text{ext}}] \le \mathbb{E}[R|a_{\text{self}}]$. The $Q$-values for the actions are:
\begin{align}
    Q^\pi(s, a_{\text{self}}) &= \mathbb{E}[R|a_{\text{self}}] \quad (\text{since } C=0) \\
    Q^\pi(s, a_{\text{ext}}) &= \mathbb{E}[R|a_{\text{ext}}] - \lambda c
\end{align}
Subtracting these yields $Q^\pi(s, a_{\text{ext}}) - Q^\pi(s, a_{\text{self}}) \le -\lambda c$. Since $\lambda, c > 0$, we have strict dominance: $Q^\pi(s, a_{\text{ext}}) < Q^\pi(s, a_{\text{self}})$. This implies strictly negative advantage:
\begin{align}
    \label{eq:neg_adv}
    A^\pi(s, a_{\text{ext}}) < 0 \quad \text{and} \quad A^\pi(s, a_{\text{self}}) > 0.
\end{align}

\textbf{2. Constraint Contradiction.}
For any candidate policy $\pi' \in \mathcal{T}_{\text{eff}}$, let $\Delta p = \pi'(a_{\text{ext}}) - \pi(a_{\text{ext}})$.
The \textit{Escalation Intent} condition requires $\Delta p > 0$.
The \textit{Policy Improvement} condition requires $\sum_{a} \pi'(a) A^\pi(s, a) > 0$. Using the identity $\sum \pi(a) A^\pi(s, a) = 0$, we expand this summation:
\begin{align}
    \mathbb{E}_{\pi'}[A^\pi] &= \sum_{a} (\pi'(a) - \pi(a)) A^\pi(s, a) \\
    &= \Delta p \cdot A^\pi(s, a_{\text{ext}}) \notag \\
    &\quad + (-\Delta p) \cdot A^\pi(s, a_{\text{self}}) \\
    &= \Delta p \underbrace{\left( A^\pi(s, a_{\text{ext}}) - A^\pi(s, a_{\text{self}}) \right)}_{\text{Strictly Negative } \delta}
\end{align}
From Step 1, the term $\delta$ is strictly negative. For the total expectation to be positive ($\Delta p \cdot \delta > 0$), we must have $\Delta p < 0$.

\textbf{Conclusion.}
The condition for Policy Improvement ($\Delta p < 0$) directly contradicts the condition for Escalation Intent ($\Delta p > 0$). Thus, no policy $\pi'$ exists that satisfies both conditions simultaneously. The effective trust region is empty.
\end{proof}

\paragraph{Implication.} This result shows that, in the regime characterized by Assumption~\ref{ass:strong_prior_app}, collapse is not merely an optimizer accident. It follows from the combination of sparse outcome supervision, positive escalation cost, and the router's failure to extract reward-improving information from the expert trace. This necessitates the \textit{Cognitive Decoupling} proposed in our method, effectively separating the routing decision from the trace execution.

\paragraph{Remark on Generalization.}
It is important to note that while our theoretical derivation assumes a binary action space ($a_{self}$ vs. $a_{ext}$) for clarity, the conclusion generalizes to hierarchical multi-model settings. The decision process in a tiered candidate pool (e.g., 4B $\to$ 30B $\to$ 235B) can be decomposed into a sequence of recursive binary decisions—essentially determining whether to escalate to the next tier of capability. Therefore, the structural validity of the Trust Region Collapse phenomenon remains applicable to the broader selection problem.

\section{Theoretical Analysis: Soft Anchor as a Trust Region}
\label{app:prof_prior}

In this appendix, we formally establish the mathematical equivalence between the Soft Anchor mechanism and a constrained optimization problem solved via Lagrange multipliers.

\paragraph{Entropy Regularization term $\mathcal{H}(\pi_\theta)$.}
In our optimization objective, the term $\mathcal{H}(\pi_\theta)$ denotes the Shannon entropy of the stochastic policy $\pi_\theta(\cdot|x)$ over the discrete action space of candidate models $\mathcal{M}$. Formally, for a given input query $x$, it is defined as:
\begin{equation}
    \mathcal{H}(\pi_\theta) = - \sum_{m \in \mathcal{M}} \pi_\theta(m|x) \log \pi_\theta(m|x).
\end{equation}
While entropy maximization is commonly used in Reinforcement Learning to encourage exploration, we strictly employ a negative coefficient $-\beta$ to penalize high-entropy states. This mechanism acts as an \textit{entropy contraction} constraint: it incentivizes the router to minimize decision uncertainty and converge towards a sharp, deterministic selection of the optimal expert, preventing the policy from remaining in an indecisive probabilistic spread at convergence.

\subsection{Problem Formulation via Lagrange Multipliers}

We formulate the router's objective as a dual-goal optimization problem: we seek to maximize the expected reward while simultaneously minimizing decision uncertainty (decision sharpening), subject to a trust-region constraint anchored to the reference policy $\pi_{\text{ref}}$.

Formally, this is expressed as:
\begin{align}
\label{eq:constrained_opt}
    \max_{\pi} \quad & \mathbb{E}_{m \sim \pi} [r(m, x)] - \beta \mathcal{H}(\pi) \\
    \text{s.t.} \quad & \mathbb{D}_{\text{KL}}(\pi(\cdot \mid x) \,||\, \pi_{\text{ref}}(\cdot \mid x)) \le \delta \nonumber
\end{align}
where $\mathcal{H}(\pi)$ is the entropy term we wish to penalize to ensure decisiveness, and $\delta$ defines the radius of the safe trust region.

We solve this constrained problem using the method of Lagrange multipliers. The Lagrangian $\mathcal{L}(\pi, \lambda)$ is:
\begin{equation}
\begin{split}
    \mathcal{L}(\pi, \lambda) &= \mathbb{E}_{\pi}[r(m, x)] - \beta \mathcal{H}(\pi) \\
    &\quad - \lambda \left( \mathbb{D}_{\text{KL}}(\pi \,||\, \pi_{\text{ref}}) - \delta \right).
\end{split}
\end{equation}
For simplicity, and to align the regularization strength, we assume a unified coefficient $\lambda = \beta$. The objective effectively becomes maximizing:
\begin{equation}
    \mathcal{J}(\pi) = \mathbb{E}_{\pi}[r(m, x)] - \beta \mathcal{H}(\pi) - \beta \mathbb{D}_{\text{KL}}(\pi \,||\, \pi_{\text{ref}}).
\end{equation}

\subsection{Equivalence to Reward Shaping}

To demonstrate that our implementation (Eq.~\ref{eq:theoretical_objective} in the main text) optimizes this Lagrangian, we expand the KL divergence term using the identity $\mathbb{D}_{\text{KL}}(\pi \,||\, \pi_{\text{ref}}) = -\mathcal{H}(\pi) - \mathbb{E}_{\pi}[\log \pi_{\text{ref}}]$. Substituting this into the objective:

\begin{align}
    \mathcal{J}(\pi) &= \mathbb{E}_{\pi}[r] - \beta \mathcal{H}(\pi) \nonumber \\
    &\quad - \beta \left( -\mathcal{H}(\pi) - \mathbb{E}_{\pi}[\log \pi_{\text{ref}}] \right) \nonumber \\
    &= \mathbb{E}_{\pi}[r] - \beta \mathcal{H}(\pi) + \beta \mathcal{H}(\pi) \\
    &\quad+ \beta \mathbb{E}_{\pi}[\log \pi_{\text{ref}}] \nonumber \\
    &= \mathbb{E}_{\pi} \left[ r(m, x) + \beta \log \pi_{\text{ref}}(m \mid x) \right].
\end{align}

\textbf{Conclusion:} 
This derivation proves that simply augmenting the reward with the log-likelihood of the reference policy (i.e., maximizing $\mathbb{E}[r + \beta \log \pi_{\text{ref}}]$) is mathematically equivalent to solving the constrained optimization problem in Eq.~\ref{eq:constrained_opt}. 
The implementation elegantly cancels out the intrinsic entropy bonus typically associated with KL constraints (since $-D_{KL}$ includes $+\mathcal{H}$), thereby implicitly enforcing the entropy penalty required for decision sharpening.
\section{Data Construction Details}
\label{app:data_construction}

In this section, we provide a comprehensive breakdown of our data synthesis pipeline, including the criteria for difficulty assessment, the implementation of Soft Supervision, and the generation of reasoning trajectories.

\subsection{SFT Dataset Construction}
\label{sec:sft_dataset}

\paragraph{Data Sources and Model Pool.}
We initially curated a collection of approximately 190k raw mathematical problems from high-quality open-source reasoning datasets, specifically \texttt{OpenMathReasoning}~\cite{moshkov2025aimo2} and \texttt{DeepMath}~\cite{he2026deepmathk}. 
Our candidate routing pool $\mathcal{M}$ consists of three models with escalating scale: $\{m_{\text{4B}}, m_{\text{30B}}, m_{\text{235B}}\}$, corresponding to the \textbf{Qwen3-Thinking-2507} series (4B, 30B-A3B, and 235B-A22B, respectively)~\cite{yang2025qwen3technicalreport}. 
Distinct from this pool, we utilize \textbf{Qwen3-1.7B} exclusively as a \textit{Calibration Probe} ($m_{\text{probe}}$) to identify queries at the lower bound of difficulty. We set the threshold $\tau=0.8$.

\paragraph{Expert Profiling and Filtering.}
For each query $x$, we estimate the ground-truth pass rate $\kappa(m, x)$ via $N=5$ Monte Carlo rollouts. The profiling process begins with the probe $m_{\text{probe}}$. If $\kappa(m_{\text{probe}}, x) \ge \tau$, the gating term is set to $\mathcal{G}(x)=1$, identifying the query as trivial. Under our cost-capability assumption, these samples are assigned to $m_{\text{best}}=$Qwen3-4B-Thinking-2507, the globally cheapest model in $\mathcal{M}$, to minimize inference overhead.

For non-trivial queries ($\mathcal{G}(x)=0$), we profile the remaining candidates. The \textbf{most economical expert} $m_{\text{best}}$ is determined as the cheapest model satisfying $\kappa(m, x) \ge \tau$ (Eq.~\ref{eq:m_best}). If no model in $\mathcal{M}$ meets the reliability threshold, we default to $m_{\text{best}} = m_{\text{235B}}$ to prioritize correctness.

\paragraph{Soft Target Construction}
To implement the \textbf{Soft Supervision} defined in Eq.~\ref{eq:unified_target}, 
we construct the soft target support from the most economical qualified experts. 
For non-trivial instances where at least two candidate models satisfy the 
reliability threshold, we set \(\mathcal{M}_{\text{top2}}\) to the two cheapest 
qualified models and distribute the target probability uniformly over them. 
This soft target avoids forcing the router to imitate a single arbitrary expert 
when multiple models are empirically viable.

For high-difficulty instances where the qualified support degenerates to the 
strongest expert, the model-label target becomes a singleton,
\[
\mathcal{Q}(m|x)=\mathbb{I}[m=m_{\text{235B}}],
\]
rather than an expanded set with duplicate model labels. To strengthen supervision 
for these extreme cases, we apply \textbf{Multi-path Reasoning Distillation}: 
for the same query and the same target model label \(m_{\text{235B}}\), we 
synthesize two distinct reasoning trajectories 
\((y_{\text{CoT}}^{(1)}, y_{\text{CoT}}^{(2)})\). This duplicates the training 
signal at the rationale level, not at the model-label level. By exposing the 
router to diverse analytical patterns associated with the strongest expert, 
the router learns to recognize extreme query complexity from multiple reasoning 
perspectives, improving the robustness of its query-analysis behavior for the 
most challenging instances.

Equivalently, this procedure can be viewed as upweighting extreme-difficulty 
instances in the SFT corpus while keeping the routing target distribution 
\(\mathcal{Q}(m|x)\) mathematically well-defined.

\textit{Prevention of Label Inflation.} The inclusion of $m_{\text{probe}}$ for initial filtering is a critical design choice. Since $m_{\text{probe}} \notin \mathcal{M}_{\text{pool}}$, trivial queries (solvable by a 1.7B model) would otherwise be assigned to Qwen3-4B-Thinking-2507. Under our setting, this would shift the target probability toward Qwen3-30B-A3B-Thinking-2507, leading to "compute waste" where simple logic is handled by mid-to-large tier models. Our probe-based filtering ensures the router focuses strictly on the appropriate reasoning complexity.

\paragraph{Reasoning Rationale Synthesis.}
To equip the router with the intermediate reasoning capabilities described in Section~\ref{sec:stage1}, we synthesized high-quality reasoning rationales (CoT) for the valid $(x, m_{\text{best}})$ pairs. We first utilized \textbf{Gemini-3-Pro}~\cite{gemini-3-pro} to generate a reasoning exemplar. Using this as a 1-shot seed, we prompted \textbf{GPT-OSS-120B}~\cite{Agarwal2025gptoss120bgptoss20bMC} to synthesize reasoning chains for the entire filtered dataset. This distillation process ensures that the router's analysis phase is grounded in sophisticated logical depth. The final SFT dataset comprises approximately 270k samples, encompassing query analysis traces and their corresponding soft-target routing labels.

\subsection{RL Dataset Construction}

For the Reinforcement Learning stage, we constructed a dataset designed to test generalization and robustness. We sampled a total of 50k problems from \texttt{OpenMathReasoning}:
\begin{itemize}
    \item \textbf{Seen Set (30k):} Problems re-sampled from the SFT distribution to prevent catastrophic forgetting.
    \item \textbf{Unseen Set (20k):} Novel problems disjoint from the SFT data to encourage generalization.
\end{itemize}
Similar to the SFT stage, we determined the difficulty reference by running \texttt{Qwen3-4B-Thinking-2507} and \texttt{Qwen3-30B-A3B-Thinking-2507} for 5 iterations each. Crucially, we treat the largest expert \texttt{Qwen3-235B-A22B-Thinking-2507} as the performance ceiling of our system, assigning it a constant pass rate of $\kappa(\text{Qwen3-235B-A22B-Thinking-2507}, x) = 1.0$ for all $x$. The reference ground truth was assigned to the smallest model achieving a $\ge \tau$ pass rate. This reference is used to compute the reward signal during GRPO training, rewarding the model for selecting the most efficient effective model.

We observed that the initial distribution was heavily skewed toward the smallest model (\texttt{Qwen3-4B-Thinking-2507}), which could bias the policy toward trivial cost-saving behaviors. To address this, we performed random downsampling on this class to ensure a balanced distribution. The final curated dataset consists of approximately 35k samples.

\section{Training Details}
\label{app:train_details}

We detail the hyperparameters and configurations for both the Supervised Fine-Tuning (SFT) and Reinforcement Learning (RL) stages below.

\paragraph{SFT Configuration.}
For the SFT stage, we set the learning rate to $2\times 10^{-5}$ with a warmup ratio of $0.1$. The models were trained for $5$ epochs.

\paragraph{RL Configuration.}
Our reinforcement learning framework is implemented using the VeRL~\cite{sheng2024hybridflow}, employing the standard GRPO algorithm~\cite{shao2024deepseekmathpushinglimitsmathematical}. We set the training batch size to $128$ and the maximum sequence length to $10,000$. For each query, we collected $16$ rollouts. Notably, we set the KL divergence coefficient to $0$. We train the backbone for 40 steps.

We also set the failure penalty $R_{\text{fail}}$ to $-3.0$. Additionally, we calibrate the cost–performance trade-off in Eq.~\ref{eq:utility_metric} by setting the hyperparameter $\alpha = 0.6$. 

\subsection{Reward Function Details}
\label{app:reward_details}

In Section~\ref{sec:stage2}, we introduced a piecewise success reward $R_{\text{succ}}$ to balance safety and efficiency. Here we provide the exact formulation. Let $m_{\text{best}}$ be the most economical model that meets the capability threshold $\tau$ (as defined in Eq.~\ref{eq:m_best}). We distinguish two regimes based on the selected model $m$'s cost $C(m)$:

\paragraph{Risk Zone ($C(m) < C(m_{\text{best}})$).}
In this regime, the selected model is cheaper than the empirically optimal safe choice. Although it may correctly answer the query (stochastically), it is statistically less reliable. To avoid "reward hacking"—where the model overfits to accidental successes or noise in the verification function—we clamp the success reward to a conservative constant:
\begin{equation}
    R_{\text{succ}} = 0.6.
\end{equation}

\paragraph{Efficiency Zone ($C(m) \ge C(m_{\text{best}})$).}
In this regime, the model is qualified but potentially over-expensive. To encourage the router to prefer $m_{\text{best}}$ over strictly stronger but wasteful models, we introduce a linear cost penalty:
\begin{equation}
    \resizebox{1\linewidth}{!}{$\displaystyle
        R_{\text{succ}} = \max \left( 1 - \frac{C(m) - C(m_{\text{best}})}{\delta_{\text{max}}} \cdot (1 - \gamma), \ \gamma \right),
    $}
\end{equation}
where $\delta_{\text{max}}$ represents the maximum cost spread in the model pool, and $\gamma$ (set to 0.2 in our experiments) ensures a minimum positive signal for all correct answers.

\subsection{Sparse Outcome Reward for Ablation}
\label{app:outcome}

To rigorously evaluate the contribution of our continuous, expectation-based reward shaping, we constructed a baseline using a discrete, sparse outcome reward signal. Formally, for a given query $x$, let $m_{\text{best}}$ denote the optimal expert defined in Eq.~(\ref{eq:m_best}) (i.e., the most economical model satisfying the capability threshold). We define the sparse reward function $R_{\text{sparse}}(m, x)$ as follows:

\begin{equation}
    R_{\text{sparse}}(m, x) = 
    \begin{cases} 
    1.0, & \text{if } m = m_{\text{best}}, \\
    0.5, & \text{if } C(m) > C(m_{\text{best}}), \\
    0.0, & \text{if } C(m) < C(m_{\text{best}}).
    \end{cases}
\end{equation}

This design imposes a strict step-function logic:
\begin{itemize}
    \item \textbf{Optimal Match ($1.0$):} The router receives maximum reward only if it selects the exact optimal model $m_{\text{best}}$.
    \item \textbf{Inefficient Success ($0.5$):} If the router selects a model that is more capable but more expensive than $m_{\text{best}}$, it receives partial credit (0.5). This acknowledges correctness but penalizes computational waste.
    \item \textbf{Failure/Risk ($0.0$):} If the router selects a model cheaper than $m_{\text{best}}$ (which is presumed unqualified), it receives zero reward, effectively penalizing the risk of incorrectness.
\end{itemize}

\section{Evaluation Details}
\label{app:eval_details}

\paragraph{Infrastructure and Deployment.} All experiments were conducted on a high-performance computational node equipped with $8 \times$ NVIDIA B200 GPUs. We deployed all external candidate models locally using the vLLM inference engine to ensure efficient high-throughput serving.

\paragraph{Generation Configuration.} We set specific maximum generation limits to accommodate the reasoning depth of different models:
\begin{itemize}
    \item \textbf{Candidate Models:} The maximum generation length was set to 32,768 tokens for Qwen3-4B and Qwen3-30B, and 60,000 tokens for the Qwen3-235B to prevent truncation of long chain-of-thought reasoning.
    \item \textbf{Router Model:} The router's generation limit was set to 12,000 tokens.
\end{itemize}

\paragraph{Robustness Protocol.} To mitigate variance and ensure statistical significance on benchmarks with limited sample sizes, we adopted a multi-run averaging strategy. We report the average performance over 16 independent runs for smaller datasets (HMMT, AIME 2024, and AIME 2025) and 8 independent runs for AMC and SMT 2025. We refer to Qwen3-235B as the strongest-expert reference and the strongest fixed-expert baseline.

\paragraph{Answer Verification.}
To evaluate correctness, we utilized math-verify\footnote{\url{https://github.com/huggingface/Math-Verify}} to parse the model outputs and determine if the generated answers were consistent with the ground truth.

\section{Intuitive Analysis on the Failure of Multi-round Baselines}
\label{sec:router_fail}

While Section~\ref{sec:collapse_observation} characterizes a sufficient regime for Trust Region Collapse, we here offer a complementary intuitive explanation.

Existing multi-round agents, such as Router-R1~\cite{zhang2025router}, are largely validated on Knowledge QA tasks. In such contexts, the primary function of the external expert is to supply missing factual information. For a router, assimilating a specific fact (e.g., an entity name) is structurally simple and aligns well with standard token prediction. However, the mathematical domain imposes a fundamentally higher cognitive burden: the external expert must provide a rigorous, multi-step \textbf{reasoning chain}. Compared to factual retrieval, these complex logical traces are highly Out-Of-Distribution (OOD) for a smaller router. The router often fails to ``comprehend'' this alien logic, perceiving the long expert trace as noise rather than a helpful guide, which leads it to revert to its own internal generation.

This intuition is strongly corroborated by empirical observations from the xRouter~\cite{qian2025xroutertrainingcostawarellms} study. Notably, the authors reported that they failed to successfully train the policy using the Qwen3 family, explicitly attributing this failure to the model's ``excessively strong pre-training priors,'' and were consequently forced to revert to Qwen2.5. This independent finding serves as a powerful validation of our \textbf{Strong Prior Dominance} hypothesis: as base models become stronger (possessing more robust internal priors), they become paradoxically more resistant to integrating external reasoning, thereby accelerating the Trust Region Collapse.

\section{Comparison with LLM-as-a-Router}
\label{app:comparison_llm_router}

\begin{table*}[!t]
\centering
\renewcommand{\arraystretch}{1.2}
\tabcolsep=0.008\linewidth
\resizebox{\linewidth}{!}{
\begin{tabular}{
l
cc cc cc cc cc cc cc cc cc cc 
}
\toprule
\multirow{2}{*}{\textbf{Model}} &
\multicolumn{2}{c}{\textbf{HMMT Nov 25}} &
\multicolumn{2}{c}{\textbf{AIME25}} &
\multicolumn{2}{c}{\textbf{OlympiadBench}} &
\multicolumn{2}{c}{\textbf{AIME-24}} &
\multicolumn{2}{c}{\textbf{AMC}} &
\multicolumn{2}{c}{\textbf{MATH500}} &
\multicolumn{2}{c}{\textbf{SMT2025}} &
\multicolumn{2}{c}{\textbf{Avg.}}  \\
\cmidrule(lr){2-3}
\cmidrule(lr){4-5}
\cmidrule(lr){6-7}
\cmidrule(lr){8-9}
\cmidrule(lr){10-11}
\cmidrule(lr){12-13}
\cmidrule(lr){14-15}
\cmidrule(lr){16-17}
 & \textbf{Acc. ($\uparrow$)} & \textbf{Cost ($\downarrow$)} & 
   \textbf{Acc. ($\uparrow$)} & \textbf{Cost ($\downarrow$)} & 
   \textbf{Acc. ($\uparrow$)} & \textbf{Cost ($\downarrow$)} & 
   \textbf{Acc. ($\uparrow$)} & \textbf{Cost ($\downarrow$)} & 
   \textbf{Acc. ($\uparrow$)} & \textbf{Cost ($\downarrow$)} & 
   \textbf{Acc. ($\uparrow$)} & \textbf{Cost ($\downarrow$)} & 
   \textbf{Acc. ($\uparrow$)} & \textbf{Cost ($\downarrow$)} & 
   \textbf{Acc. ($\uparrow$)} & \textbf{Cost ($\downarrow$)}  \\
\midrule
GPT-OSS-120B & 74.4 & 0.2974 & 74.6 & 0.2996 & 80.9 & 4.510 & 88.3 & 0.1877 & 95.6 & 0.2285 & 98.2 & 0.4757 & 78.3 & 0.5082 & 87.7 & 0.0047 \\
Qwen3-235B-A22B-Instruct & 70.0 & 0.2380 & 72.7 & 0.2624 & 80.0 & 3.7312 & 87.1 & 0.1760 & 95.8 & 0.2029 & 97.6 & 0.6625 & 75.7 & 0.4565 & 86.8 & 0.0041 \\
\midrule
\textbf{\ours} & \textbf{80.0} & 0.4163 & \textbf{78.8} & 0.2929 & \textbf{81.8} & 5.2397 & \textbf{90.6} & 0.2881 & \textbf{96.6} & 0.3658 & \textbf{98.4} & 0.5238 & \textbf{80.4} & 0.5155 & \textbf{88.6} & 0.0055  \\
\bottomrule
\end{tabular}
}
\caption{
Comparison between \ours~(ours) and LLM-as-a-Router across multiple math benchmarks.
}
\label{tab:appendix_comparison_llm_router}
\end{table*}

In this section, we provide a detailed comparison between our method (\textbf{Ours}) and the \textbf{LLM-as-a-Router} baseline. The \textbf{LLM-as-a-Router} approach involves using LLMs directly as routers for selecting the best model for each query. We specifically compare our method with two configurations of LLM-as-a-Router: \texttt{GPT-OSS-120B} and \texttt{Qwen3-235B-A22B-Instruct}, both of which are high-capability but high-latency alternatives.

The comparison results across these benchmarks are summarized in Table~\ref{tab:appendix_comparison_llm_router}. Our method consistently outperforms the \textbf{LLM-as-a-Router} baseline, achieving higher accuracy with comparable cost. While \textbf{LLM-as-a-Router} directly uses large LLMs to decide on the routing model for each query, our method leverages a specialized routing policy to achieve stronger accuracy under a comparable cost scale, avoiding direct reliance on large LLMs as routers. This demonstrates the effectiveness of our optimized routing approach compared to using LLMs as routers.

\begin{table}[!t]
\centering
\renewcommand{\arraystretch}{1.1}
\setlength{\tabcolsep}{2pt}
\resizebox{\linewidth}{!}{
\begin{tabular}{l cc cc cc}
\toprule
\multirow{2}{*}{\textbf{Model}} & \multicolumn{2}{c}{\textbf{GPQA-Diamond}} & \multicolumn{2}{c}{\textbf{MMLU-Pro}} & \multicolumn{2}{c}{\textbf{OOD Average}} \\
\cmidrule(lr){2-3} \cmidrule(lr){4-5} \cmidrule(lr){6-7}
& \textbf{Acc. ($\uparrow$)} & \textbf{Cost $\downarrow$}  & \textbf{Acc. ($\uparrow$)}  & \textbf{Cost $\downarrow$} & \textbf{Acc. ($\uparrow$)}  & \textbf{Cost $\downarrow$} \\
\midrule
GPT-OSS-120B & 71.7 & 0.6856 & 77.6 & 1.0022 & 76.8 & 0.0012 \\
Qwen3-235B-A22B-Instruct & 72.7 & 0.3141 & 79.4 & 0.8459 & 78.5 & 0.0008 \\
\midrule
\textbf{\ours} & \textbf{73.2} & 0.9018 & \textbf{80.6} & 2.3002 & \textbf{79.5} & 0.0023 \\
\bottomrule
\end{tabular}
}
\caption{Comparison between \ours~(ours) and LLM-as-a-Router across OOD benchmarks.}
\label{tab:appendix_ood}
\end{table}

Additionally, we evaluate both methods on OOD benchmarks, as shown in Table~\ref{tab:appendix_ood}. The results show that \ours~maintains high accuracy, confirming the practical advantage of our method in diverse real-world scenarios.

\section{Generalization to Out-of-Distribution Tasks}
\label{sec:OOD}
To evaluate the cross-domain robustness of our approach, we report performance on two prominent Out-of-Distribution (OOD) benchmarks: GPQA-Diamond~\cite{rein2024gpqa} and MMLU-Pro~\cite{wang2024mmluprorobustchallengingmultitask}. Due to computational constraints, we conduct our evaluation on a randomly sampled subset of 1,200 instances from each dataset. As illustrated in Table~\ref{tab:OOD}, \ours~ consistently demonstrates superior generalization capabilities across OOD domains. Our method maintains high competitive efficacy without domain-specific fine-tuning. Specifically, \ours~ achieves an accuracy reduction of only $3.8\%$, while simultaneously yielding a significant cost saving of $40.5\%$ relative to the strongest fixed-expert baseline.

\begin{table}[!t]
\centering
\renewcommand{\arraystretch}{1.1}
\setlength{\tabcolsep}{2pt}
\resizebox{\linewidth}{!}{
\begin{tabular}{l cc cc cc}
\toprule
\multirow{2}{*}{\textbf{Model}} & \multicolumn{2}{c}{\textbf{GPQA}} & \multicolumn{2}{c}{\textbf{MMLU-Pro}} & \multicolumn{2}{c}{\textbf{Avg.}} \\
\cmidrule(lr){2-3} \cmidrule(lr){4-5} \cmidrule(lr){6-7}
& \textbf{Acc.} & \textbf{Cost}  & \textbf{Acc.}  & \textbf{Cost} & \textbf{Acc.}  & \textbf{Cost} \\
\midrule
\rowcolor[gray]{0.95} \multicolumn{7}{c}{\textit{Heuristic}} \\
Random & 71.2 & 0.696 & 79.6 & 1.927 & 78.4 & 0.002 \\
Qwen3-4B & 62.6 & 0.052 & 75.0 & 0.168 & 73.3 & 0.000 \\
Qwen3-30B & 69.7 & 0.707 & 80.3 & 1.874 & 78.8 & 0.000 \\
Qwen3-235B  & 78.8 & 1.452 & 84.0 & 3.749 & 83.3 & 0.004 \\
\midrule
\rowcolor[gray]{0.95} \multicolumn{7}{c}{\textit{Existing Routing Architectures}} \\
RouteLLM & 65.2 & 0.271 & 77.6 & 0.947 & 75.8 & 0.001 \\
FORC & 71.7 & 0.741 & 78.5 & 1.644 & 77.5 & 0.002 \\
Router-R1 & 49.5 & 0.001 & 49.6 & 0.002 & 49.6 & 0.000 \\
xRouter & 42.4 & 0.005 & 52.1 & 0.003  & 50.7 & 0.000 \\
\midrule
\textbf{\ours} & \textbf{73.2} & 0.902 & \textbf{80.6} & 2.300 & \textbf{79.5} & 0.002 \\
\bottomrule
\end{tabular}
}
\caption{Comparison across Out-of-Distribution (OOD) benchmarks. 
The best performing model is highlighted in bold.
}
\label{tab:OOD}
\end{table}

\section{Impact of Cost Sensitivity $\alpha$}
\label{sec:ablation_alpha}

\begin{table}[t]
\centering
\small
\renewcommand{\arraystretch}{1.15}
\setlength{\tabcolsep}{3.5pt} 
\resizebox{\linewidth}{!}{
\begin{tabular}{l cc cc cc cc}
\toprule
\multirow{2}{*}{\textbf{Method}} & \multicolumn{2}{c}{\textbf{HMMT-Nov}} & \multicolumn{2}{c}{\textbf{AIME25}} & \multicolumn{2}{c}{\textbf{SMT2025}} & \multicolumn{2}{c}{\textbf{Average}} \\
\cmidrule(lr){2-3} \cmidrule(lr){4-5} \cmidrule(lr){6-7} \cmidrule(lr){8-9}
& \textbf{Acc} & \textbf{Cost} & \textbf{Acc} & \textbf{Cost} & \textbf{Acc} & \textbf{Cost} & \textbf{Acc} & \textbf{Cost} \\
\midrule
\multicolumn{9}{c}{\textit{\ours~ (Ablation on $\alpha$)}} \\
\midrule
$\alpha=0.2$ & 79.0 & 0.479 & 78.5 & 0.454 & \textbf{82.3} & 0.710 & \textbf{80.4} & 0.015 \\
$\alpha=0.6$ & \textbf{80.0} & 0.416 & \textbf{78.8} & 0.293 & 80.4 & 0.516 & 79.9 & 0.011 \\
$\alpha=1.0$ & 72.3 & 0.276 & 75.0 & 0.255 & 77.1 & 0.448 & 75.3 & 0.009 \\ 
\bottomrule
\end{tabular}
}
\caption{\textbf{Ablation Study on Cost-Performance Trade-off ($\alpha$).} We evaluate \ours~ with varying sensitivity to computational cost. Increasing $\alpha$ effectively reduces inference costs but may compromise accuracy if the penalty becomes too aggressive.}
\label{tab:alpha}
\end{table}

We further examine the impact of the hyperparameter $\alpha$ (Eq.~\ref{eq:utility_metric}), which governs the sensitivity of the Soft Anchor to computational cost. As shown in Table~\ref{tab:alpha}, $\alpha$ serves as an effective control knob for navigating the accuracy-efficiency Pareto frontier. Specifically, while a lower setting ($\alpha=0.2$) yields the highest accuracy $80.4\%$, increasing $\alpha$ to 0.6 achieves a significant reduction in computational cost with only a slight performance drop (from $80.4\% \to 79.9\%$). However, an overly aggressive penalty ($\alpha=1.0$) causes the cost constraint to outweigh correctness, leading to a sharp performance decline ($75.3\%$). 

\section{Generalization to Heterogeneous Model Families}
\label{app:cross_family}

To examine whether ENTROROUTER is tied to a single model family, we further augment the original Qwen3 candidate pool with two additional experts from different model families: \texttt{R1-Distill-Llama-8B} and \texttt{GPT-OSS-20B}. Unlike the main experiments, where the candidate models follow a relatively clean within-family scaling hierarchy, this setting introduces heterogeneous models with different pre-training recipes, architectures, and reasoning behaviors. It therefore provides a more realistic test of whether the router can generalize beyond family-specific scaling patterns.

Due to the additional cost of profiling the newly added experts and constructing model-specific routing labels, we conduct this study on a subset of the training data. We otherwise follow the same training and evaluation protocol as in the main experiments. The results are reported in Table~\ref{tab:cross_family}.

\begin{table}[t]
\centering
\resizebox{\linewidth}{!}{
\begin{tabular}{l cc cc cc}
\toprule
\multirow{2}{*}{\textbf{Method}} & \multicolumn{2}{c}{\textbf{AIME25}} & \multicolumn{2}{c}{\textbf{HMMT-Nov}} & \multicolumn{2}{c}{\textbf{SMT2025}}  \\
\cmidrule(lr){2-3} \cmidrule(lr){4-5} \cmidrule(lr){6-7} 
& \textbf{Acc} & \textbf{Cost} & \textbf{Acc} & \textbf{Cost} & \textbf{Acc} & \textbf{Cost} \\
\midrule
Random & 65.0 & 0.193 & 61.3 & 0.194 & 66.3 & 0.320 \\
R1-Distill-Llama-8B & 31.9 & 0.021 & 27.3 & 0.023 & 39.9 & 0.032 \\
GPT-OSS-20B & 42.1 & 0.033 & 33.7 & 0.030 & 51.7 & 0.055 \\
\textbf{\ours} & \textbf{73.7} & 0.171 & \textbf{70.4} & 0.213 & \textbf{75.7} & 0.345 \\
\bottomrule
\end{tabular}
}
\caption{Generalization to an augmented heterogeneous candidate pool. We report accuracy and total inference cost on three mathematical reasoning benchmarks.}
\label{tab:cross_family}
\end{table}

\ours~consistently improves accuracy over both fixed-model selection and random routing. In particular, it outperforms random routing by 8.7, 9.1, and 9.4 percentage points on AIME25, HMMT-Nov, and SMT25, respectively, while maintaining a comparable cost level. These results indicate that the learned routing policy is not restricted to a single-family scaling hierarchy, but can also adapt to candidate pools containing models with distinct training recipes and reasoning behaviors.

\begin{table}[!h]
\centering
\resizebox{\linewidth}{!}{
\begin{tabular}{l cc cc cc}
\toprule
\multirow{2}{*}{\textbf{Perturbation}} & \multicolumn{2}{c}{\textbf{AIME25}} & \multicolumn{2}{c}{\textbf{HMMT-Nov}} & \multicolumn{2}{c}{\textbf{SMT2025}}  \\
\cmidrule(lr){2-3} \cmidrule(lr){4-5} \cmidrule(lr){6-7} 
& \textbf{Acc} & \textbf{Cost} & \textbf{Acc} & \textbf{Cost} & \textbf{Acc} & \textbf{Cost} \\
\midrule
0\%  & 78.8 & 0.293 & 80.0 & 0.416  & 80.4 & 0.516 \\
10\% & 79.6 & 0.357 & 80.2 & 0.411 & 80.0 & 0.591 \\
\bottomrule
\end{tabular}
}
\caption{Robustness to noisy offline estimates. We inject noise into 10\% of the training data by perturbing both Stage-I target labels and Stage-II pass-rate estimates.}
\label{tab:noise_robustness}
\end{table}

Importantly, the heterogeneous setting weakens the simple monotonic relationship between model cost and capability assumed in a clean within-family pool. Models from different families may exhibit different specialization patterns, and their relative performance can vary substantially across benchmarks. Therefore, routing can no longer be reduced to selecting models according to scale or price alone. The gains in Table~\ref{tab:cross_family} suggest that ENTROROUTER can leverage empirical capability estimates and query-dependent routing signals to identify effective experts even when the candidate pool does not follow a strictly monotonic cost--capability ordering.

\section{Robustness to Noisy Offline Estimates}
\label{app:noise_robustness}

ENTROROUTER relies on offline capability estimates $\kappa(m,x)$ to construct the SFT targets and the RL reward. Although these estimates are obtained through Monte Carlo rollouts, this profiling is a one-time offline cost and does not affect deployment: at inference time, the router performs a single forward pass and requires no online rollouts. This preserves the scalability of the single-round routing framework.

To examine the robustness of ENTROROUTER to imperfect offline estimates, we conduct a joint-noise stress test on the training data. Specifically, we randomly perturb 10\% of the training instances in both stages: (1) for Stage I, we shift the target model label to an adjacent capability tier when available; and (2) for Stage II, we perturb the offline pass-rate estimates $\kappa(m,x)$ by $\pm 0.2$. This setting is intentionally stronger than an isolated Stage-II perturbation, since it simultaneously corrupts the supervised initialization targets and the RL-stage capability estimates.

As shown in Table~\ref{tab:noise_robustness}, ENTROROUTER maintains stable accuracy under this joint-noise perturbation. The accuracy changes are small across all three benchmarks, indicating that the learned router is not overly sensitive to moderate corruption in the offline supervision signals. We also observe a mild increase in cost on AIME25 and SMT25, suggesting that noisy estimates can make the router slightly more conservative by assigning more queries to stronger experts. Overall, these results show that ENTROROUTER remains robust under imperfect offline profiling while preserving its core accuracy--cost trade-off.